\documentclass{article}

\usepackage{microtype}
\usepackage{graphicx}
\usepackage{subcaption}
\usepackage{booktabs} % for professional tables
\usepackage{pifont}
\usepackage{array}
\newcommand{\xmark}{\ding{55}}

\usepackage{comment}

\usepackage{hyperref}

\usepackage[preprint]{icml2026}

\usepackage{amsmath}
\usepackage{amssymb}
\usepackage{mathtools}
\usepackage{amsthm}

\usepackage[capitalize,noabbrev]{cleveref}

\theoremstyle{plain}
\newtheorem{theorem}{Theorem}[section]

\newtheorem{lemma}[theorem]{Lemma}

\theoremstyle{definition}

\theoremstyle{remark}

\newcommand{\ours}{\textsc{SPARK}}
\usepackage[textsize=tiny]{todonotes}

\icmltitlerunning{Submission and Formatting Instructions for ICML 2026}

\begin{document}

\twocolumn[
    \icmltitle{\ours: \underline{S}tochastic \underline{P}ropagation via \underline{A}ffinity-guided \underline{R}andom wal\underline{K} for \textit{training-free} unsupervised segmentation}
%   \icmltitle{Diffusion Affinity Isn’t Enough: Zero-Shot Segmentation via
% Local Neighborhood and Random-Walk Propagation}



  \icmlsetsymbol{equal}{*}

\begin{icmlauthorlist}
  \icmlauthor{Kunal Mahatha}{ets,ills}
  \icmlauthor{Jose Dolz}{ets,ills}
  \icmlauthor{Christian Desrosiers}{ets,ills}
\end{icmlauthorlist}

\icmlaffiliation{ets}{LIVIA, École de technologie supérieure (ÉTS), Montréal, Canada}
\icmlaffiliation{ills}{International Laboratory on Learning Systems (ILLS)}
% \icmlaffiliation{}{McGILL - ETS - MILA - CNRS - Université Paris-Saclay - CentraleSupélec, Canada}

\icmlcorrespondingauthor{Kunal Mahatha}{kunal.mahatha.1@ens.etsmtl.ca}


  \icmlkeywords{Machine Learning, ICML}

  \vskip 0.3in
]



\printAffiliationsAndNotice{}  % no special notice (required even if empty)

\begin{abstract}
We argue that existing training-free segmentation methods rely on an implicit and limiting assumption, that segmentation is a spectral graph partitioning problem over diffusion-derived affinities. Such approaches, based on global graph partitioning and eigenvector-based formulations of affinity matrices, suffer from several fundamental drawbacks, they require pre-selecting the number of clusters, induce boundary oversmoothing due to spectral relaxation, and remain highly sensitive to noisy or multi-modal affinity distributions. Moreover, many prior works neglect the importance of local neighborhood structure, which plays a crucial role in stabilizing affinity propagation and preserving fine-grained contours.
To address these limitations, we reformulate training-free segmentation as a stochastic flow equilibrium problem over diffusion-induced affinity graphs, where segmentation emerges from a stochastic propagation process that integrates global diffusion attention with local neighborhoods extracted from stable diffusion, yielding a sparse yet expressive affinity structure. Building on this formulation, we introduce a Markov propagation scheme that performs random-walk-based label diffusion with an adaptive pruning strategy that suppresses unreliable transitions while reinforcing confident affinity paths.
Experiments across seven widely used semantic segmentation benchmarks demonstrate that our method achieves state-of-the-art zero-shot performance, producing sharper boundaries, more coherent regions, and significantly more stable masks compared to prior spectral-clustering-based approaches. Our project page is publicly available at: \url{https://kunal-mahatha.github.io/spark/}.
\end{abstract}

\section{Introduction}

Zero-shot image segmentation has become an appealing alternative to fully supervised pipelines, enabling segmentation of arbitrary concepts without pixel-level labels or task-specific finetuning. With the rise of diffusion models, recent work has shown that their internal self-attention maps encode rich semantic structure that can be repurposed as training-free affinity  ~\cite{couairon2024diffcut}. This has led to diffusion-based segmentation methods that operate without any training, prompts, or supervision.

\begin{figure}[t]
    \centering
    \includegraphics[width=.98\linewidth]{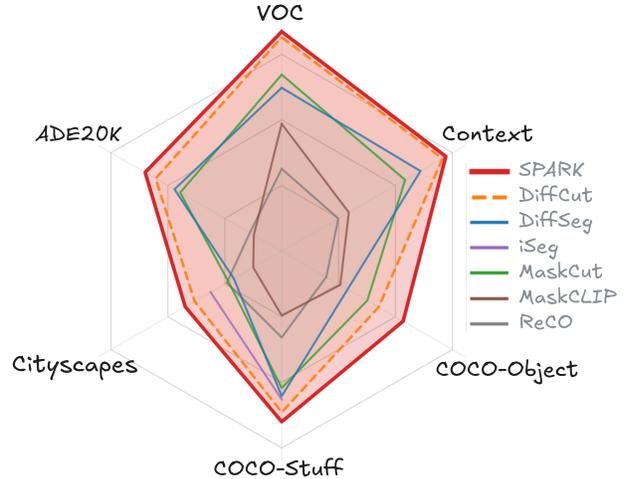}

    \vspace{-5pt}
    \caption{
    \textbf{Radar plot of mIoU (\%) on six benchmarks.} \ours{}, which combines global diffusion-based semantic affinity with local random-walk propagation, consistently outperforms prior training-free methods across all datasets, demonstrating stronger cross-domain generalization and more reliable semantic propagation without supervision.}
    \label{motovation}

    \vspace{-5pt}
\end{figure}

Many of these methods \cite{couairon2024diffcut, fan2022simple, tian2024diffuse, ivanova2024unsupervised} formulate segmentation as a global graph partitioning problem over diffusion-derived affinities, typically solved via spectral or eigenvector-based objectives. While intuitive, this formulation introduces fundamental limitations. Eigenvectors of the graph Laplacian are highly unstable: small perturbations in diffusion attention, common due to stochastic sampling and multi-modal semantics, can induce large rotations in eigenspaces, leading to fragmented or inconsistent partitions. Moreover, spectral relaxation favors low-frequency global modes, inherently oversmoothing object boundaries and suppressing fine-scale structures. Balanced partition constraints further bias solutions toward coarse global splits, which is poorly matched to the highly unbalanced, multi-object nature of natural images. At high resolution, approximate eigensolvers exacerbate these issues, further degrading segmentation fidelity.

These observations lead to a critical insight: \emph{diffusion affinity alone is not enough}. While diffusion attention provides strong long-range semantic correspondence, it lacks the \emph{local geometric grounding} required for spatially coherent segmentation. Global affinity alone easily propagates noise, ignores boundary cues, and struggles to maintain local consistency, limitations that spectral cuts further amplify. What is needed is a mechanism that (\emph{i}) leverages the semantic benefits of diffusion attention, (\emph{ii}) incorporates local structural information, and (\emph{iii}) avoids the brittleness of eigenvalue decomposition.

To address these limitations, we introduce \ours{}, a simple yet effective training-free segmentation framework that leverages stochastic-flow propagation over a pixel affinity graph. By selectively concentrating probability mass on the most relevant patch interactions, \ours{} induces a naturally sparse and focused attention structure. Our method applies an iterative expansion–inflation procedure, inspired by Markov Clustering~\cite{vondongen2008graph}, to produce flow-preserving clusters that correspond to coherent object segments, without relying on eigenvectors, training, or prompts. Moreover, to jointly capture long-range semantic interactions and enforce spatial smoothness with boundary sensitivity, \ours{} combines pixel affinities derived from both diffusion self-attention and local neighborhood.

These results empirically demonstrate that reframing segmentation as a stochastic flow equilibrium problem leads to more robust and stable region discovery than spectral graph partitioning, validating the proposed dynamical systems perspective for training-free segmentation.

\noindent\textbf{Our contributions are:}
\begin{itemize}
    \item We identify fundamental limitations of \emph{spectral and eigenvector-based graph partitioning formulations} when applied to diffusion-derived affinity graphs, including instability to noise, boundary oversmoothing, and bias toward coarse global partitions.
    
    \item In light of these limitations, we propose \ours{}, a training-free segmentation framework that fuses \emph{global diffusion affinity} with \emph{local spatial structure} and replaces spectral relaxation with a \emph{stochastic Markov-flow propagation} mechanism that naturally induces sparse, coherent segmentations.
    
    \item Comprehensive experiments across six popular benchmarks showcase the superiority of \ours{}, which achieves state-of-the-art, training-free segmentation performance as demonstrated in Figure~\ref{motovation}, substantially improving boundary accuracy, region coherence, and segmentation stability over prior diffusion-based methods.
\end{itemize}

\section{Related Works}
A large body of work formulates object detection and unsupervised segmentation as a graph partitioning problem over patch or pixel affinities. Methods such as LOST \cite{simeoni2021localizing}, TokenCut \cite{wang2023tokencut}, FOUND \cite{simeoni2023unsupervised}, and MaskCut \cite{wang2023cut} construct graphs from self-supervised ViT features and apply spectral or normalized cut objectives to identify salient objects or partition images into regions. While effective for coarse grouping, these approaches rely on eigenvector-based graph partitioning, which introduces balanced partition bias, boundary oversmoothing, and sensitivity to noise, making them poorly suited for fine-grained semantic segmentation.

Recent works has shown that diffusion models encode rich semantic correspondence in their attention maps, enabling training-free zero-shot segmentation. DiffSeg \cite{tian2024diffuse}, EmerDiff \cite{namekata2024emerdiff}, and related methods extract self- or cross-attention from Stable Diffusion to obtain pixel or patch affinities. Several approaches, including DiffCut \cite{couairon2024diffcut} and \cite{fan2022simple,ivanova2024unsupervised}, further formulate segmentation as a global graph partitioning problem over these diffusion-derived affinities, typically solved via spectral or eigenvector-based objectives. While diffusion attention provides strong long-range semantic cues, such spectral formulations struggle to translate them into spatially coherent segmentation due to eigenvector instability and the lack of local geometric grounding.

Moving beyond spectral graph partitioning, we introduce the first training-free, fully unsupervised segmentation framework based on stochastic flow propagation~\cite{vondongen2008graph}. While stochastic flow propagation was originally developed for graph partitioning, where expansion–inflation dynamics identify dense regions by amplifying consistent transition paths and suppressing weak affinities, we repurpose this mechanism for image segmentation. Our \ours{} method integrates diffusion-based semantic affinities with local spatial consistency within a unified stochastic flow framework, enabling boundary-aware zero-shot segmentation without training, prompts, or eigenvector computation.

\begin{figure*}[t]
    \centering
    \includegraphics[width=0.9\textwidth]{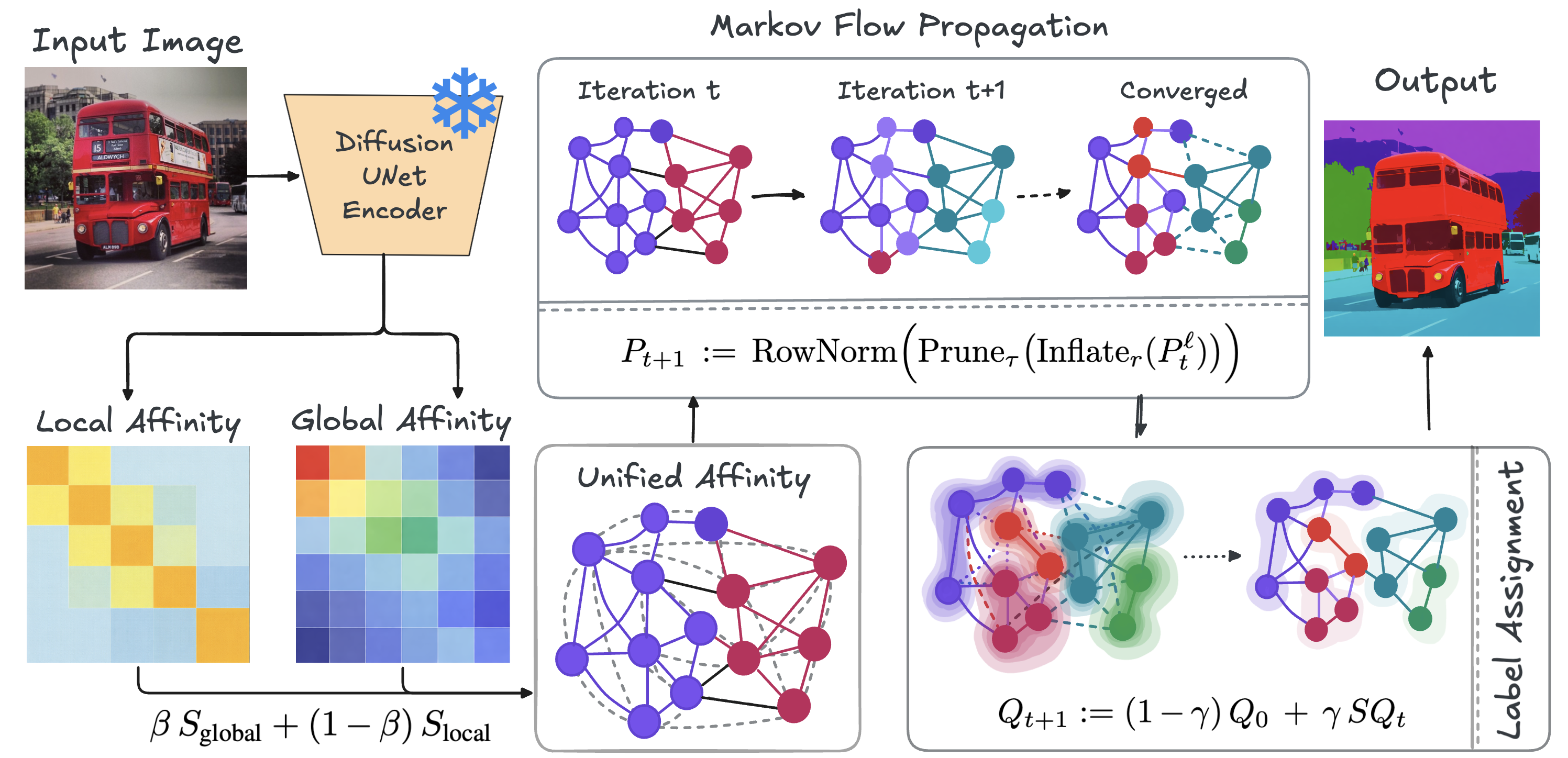}
    \caption{\textbf{Overview of our training-free unsupervised segmentation pipeline.} From an input image, we extract self-attention maps using a frozen Diffusion U-Net encoder to form a global affinity $S_{\text{global}}$ and a sparse local affinity $S_{\text{local}}$. These are normalized and fused into a unified affinity $S$, which defines a Markov random walk over the pixel graph. Iterative Markov-flow propagation drives the system toward stable flow-preserving clusters, whose stationary distribution is used for label assignment, yielding the final segmentation.}
    \label{fig:pipeline}
\end{figure*}

\section{Proposed method}

The general pipeline of \ours{} is shown in Fig.~\ref{fig:pipeline}. This pipeline is composed of three steps: 1) transition matrix construction, and 2) Markov-flow clustering, and 3) random-walk label propagation, which are detailed in the following subsections.

\subsection{Transition Matrix Construction}

Inspired by Markov Clustering~\cite{vondongen2008graph}, we formulate the unsupervised segmentation task as a stochastic flow process on a graph $\mathcal{G} = (\mathcal{V}, \mathcal{E})$.  
Each node $i \in \mathcal{V}$ corresponds to a spatial token and each edge $(i,j)\in\mathcal{E}$ is weighted by affinity $a_{i,j}$ that captures both global and local relationships.  

\mypar{Global Diffusion Affinity.}
We begin by constructing a global affinity matrix using the features from the last layer of a pre-trained Stable Diffusion U-Net \cite{rombach2021highresolution}. The feature map having a spatial resolution of $H \,{\times}\, W$, we define the $N = HW$ per-token features as $\{X_i\}_{i=1}^N$, where $X_i \in \mathbb{R}^C$.

The global semantic affinity between two tokens is computed as the inner product of their diffusion features:
\begin{equation}
[A_{\text{global}}]_{i,j}
\,:=\,
\langle X_i, X_j \rangle,% }{\max_{u,v} X_u^\top X_v},
\quad A_{\text{global}} = X X^{\top}
\label{eq:global}
\end{equation}
where $X \in \mathbb{R}^{N\times C}$ is the token feature matrix.
This global affinity enables information to propagate across distant but semantically related regions. 

\mypar{Local Spatial Affinity.}
While global semantic affinity captures long-range correspondence, it lacks spatial precision and may connect distant but visually similar regions while ignoring local image structure. To enforce local consistency and boundary sensitivity, we construct a complementary affinity matrix over an 8-connected spatial neighborhood. Specifically, for each pixel $i$ with feature vector $X_i$, we define its neighbors as
\begin{equation}
\begin{aligned}
\mathcal{N}(i) \,=\, & \big\{ j : (dh,dw)\in\{-1,0,1\}^2\setminus\{(0,0)\},\\
& \quad \mathbf{x}_j=\mathbf{x}_i+(dh,dw) \big\}.
\end{aligned}
\end{equation}
We then define the local affinity as a sparse cosine-similarity graph:
\begin{equation}
[A_{\text{local}}]_{i,j} :=
\begin{cases}
1, & i=j,\\[4pt]
\frac{\langle X_i , X_j\rangle}{\|X_i\|\,\|X_j\|} + \epsilon, & j\in\mathcal{N}(i),\\[6pt]
0, & \text{otherwise},
\end{cases}
\label{eq:local}
\end{equation}
where $\epsilon>0$ prevents degenerate zero affinities.  
This sparse local graph enforces spatial smoothness, preserves object boundaries, and stabilizes the subsequent random-walk propagation in regions where global semantic affinity may be noisy or diffuse.

\mypar{Fused Global and Local Affinity.}
To leverage both long-range semantic cues and short-range spatial structure, we combine the two affinity matrices through a convex fusion. In particular, after row-normalizing each matrix,
\[
S_{\text{global}}=\mathrm{RowNorm}(A_{\text{global}}), \quad
S_{\text{local}}=\mathrm{RowNorm}(A_{\text{local}}),
\]
we compute the final transition matrix as
\begin{equation}
S := \beta\, S_{\text{global}} 
    + (1\,{-}\,\beta)\, S_{\text{local}},
\label{eq:fused}
\end{equation}
where $\beta\in[0,1]$ balances semantic propagation and spatial coherence.  
This fused stochastic matrix forms the backbone of \ours{}: it drives the random-walk process, reinforces high-confidence semantic connections, and yields robust, spatially consistent segmentation results.

\subsection{Markov-Flow Clustering}
\label{sec:markov-flow}

Unlike classical random-walk segmentation, which requires a node-to-label generation matrix and computes an infinite-length walk distribution, our method identifies \emph{flow-preserving regions} directly from the structure of $S$. The intuition is that pixels belonging to the same object form strongly connected components under repeated stochastic propagation, while weak cross-object transitions decay over time.

To extract these components, we start with a matrix $P_0 := S$ and update this matrix iteratively with an operator consisting of four sequential steps:

\mypar{1) Expansion.}
We first propagate flow by computing an $\ell$-step transition ($\ell$-th matrix power of $P$)
\begin{equation}
\mathrm{Expand}_\ell(P) \,:=\, \underbrace{P \cdot P \cdot P\, \cdots}_{\ell \text{ times}} \,=\, P^{\ell},
\end{equation}
with $\ell \in \mathbb{N}^+$ (typically $\ell=2$ or $3$).  
This operation corresponds to a short random walk on the graph, strengthening multi-hop associations within dense regions. 

\mypar{2) Inflation.}
To sharpen the connectivity structure, we then apply an \emph{element-wise} power
\begin{equation}
\big[\mathrm{Inflate}_r(P)\big]_{i,j} \,:=\, P_{i,j}^{r},
\end{equation}
with inflation parameter $r\,{>}\,1$. This step amplifies dominant transitions while suppressing weak ones, acting as a nonlinear diffusion that enhances dense clusters.

\mypar{3) Pruning.}
Next, to promote sparsity and stabilize the emerging block structure, we prune entries smaller than a threshold $\tau$:
\[
\big[\mathrm{Prune}_\tau(P)\big]_{i,j} \, := \, \left\{
\begin{array}{ll}
P_{i,j}, & \text{if } P_{i,j} \geq \tau\\
0,  & \text{esle}
\end{array}
\right.
\]
\mypar{4) Normalization.}
Finally, we row-normalize the matrix to restore stochasticity:
\begin{equation}
\big[\mathrm{RowNorm}(P)\big]_{i,j}
\,:=\,
\frac{P_{i,j}}{\sum_{j'} P_{i,j'}}.
\end{equation}
The full update rule can be summarized as 
\begin{equation}\label{eq:update-rule}
P_{t+1}
\,:=\,
\mathrm{RowNorm}\Big(
  \mathrm{Prune}_\tau\big( \mathrm{Inflate}_r(P_t^\ell)\big)\Big).
\end{equation}

\mypar{Coarse Cluster Extraction.}
The rule of Eq. (\ref{eq:update-rule}) is applied iteratively until no further change is observed (see Supplementary materials for convergence analysis):
\[
\big\|P_{t+1} - P_t\big\|_{\infty} < \varepsilon.
\]
Upon convergence, the fixed-point matrix $P^\ast$ exhibits an approximately block-diagonal structure, where each block corresponds to a \emph{flow-preserving ergodic component} of the underlying Markov chain. Within each block, transition probabilities remain high, while transitions across blocks are progressively suppressed through repeated application of expansion, inflation, and pruning. Intuitively, a larger expansion coefficient $\ell$ emphasizes connectivity at coarser graph scales, whereas increasing the inflation parameter $r$ accelerates convergence toward a near-binary stochastic matrix $P^\ast$ that highlights densely connected regions.

Clusters are obtained from $P^\ast$  as the indexes of columns $\{j_1, \ldots, j_K\}$ having at least one non-zero entry, i.e., the graph nodes with at least one positive flow value). From a dynamical systems perspective, these nodes act as \emph{attractors} of the Markov dynamics: once flow enters such a region, it remains confined there under further propagation. Each one defines a cluster 
\begin{equation}
\mathcal{C}_k \, := \, \big\{i \, | \, j_k = \arg\max_j \, P^\ast_{i,j}\big\}
\end{equation}
This partition emerges from the stochastic flow itself, without requiring eigenvectors or pre-specified number of clusters.

\subsection{Random-Walk Label Propagation}

To refine the coarse clusters obtained in the previous step, we apply a final random-walk diffusion using the cluster indexes as seed. Denote as $C \in \{0,1\}^{N\times K}$ the pixel-to-cluster assignment matrix such that
\[
C_{i,k} = 1 \text{ if } i \in \mathcal{C}_k, \text{ else } C_{i,k} = 0.
\]
We convert $C$ into a stochastic matrix $Q_0$ via row-normalization,
\[
Q_0 := \mathrm{RowNorm}(C),
\]
and then perform the following update iteratively:
\begin{equation}
    Q_{t+1}
    := (1\,{-}\,\gamma)\,Q_0
      \, + \, \gamma\, S Q_t
    \label{eq:rw-iter}
\end{equation}
where $0\,{<}\,\gamma\,{<}\,1$ balances cluster seed consistency and spatial propagation. After convergence, the solution $Q^{\ast}$ satisfies the discrete Dirichlet boundary condition:
\begin{equation}
    (I - \gamma S)\,Q^{\ast} = (1\,{-}\,\gamma)\,Q_0.
    \label{eq:rw-solution}
\end{equation}
Finally, the label of pixel $i$ is chosen as the cluster index with highest random-walk probability:
\begin{equation}
    \hat{y}_i \, := \, \arg\max_{k} \, Q^{\ast}_{i,k}.
\end{equation}

\begin{table*}
\scriptsize
\centering
\caption{\textbf{Unsupervised segmentation results across six benchmarks.} We report zero-shot mIoU using diffusion features. Best results are shown in \textbf{bold}, and second–best are \underline{underlined}, whereas differences with second best method are indicated in \textbf{\textcolor{mydarkgreen}{green}}. \ours{} consistently outperforms prior training-free approaches, including DiffCut, and Seg4Diff on all the datasets}
\vspace{0.50em}

\label{tab:main_result}

\resizebox{\textwidth}{!}{
\def\arraystretch{1.2}
\begin{tabular}{l|ccc|ccc}
\toprule
Model & \textbf{VOC} & \textbf{Context} & \textbf{COCO-Object}
      & \textbf{COCO-Stuff-27} & \textbf{Cityscapes} & \textbf{ADE20K} \\
\midrule

ReCO \cite{shin2022reco}\textcolor{gray}{$_\text{ NeurIPS'22}$}         & 25.1 & 19.9 & 15.7 & 26.3 & 19.3 & 11.2 \\
%\addlinespace[2pt]
MaskCLIP \cite{ding2023maskclip}\textcolor{gray}{$_\text{ ICML'23}$}     & 38.8 & 23.6 & 20.6 & 19.6 & 10.0 &  9.8 \\
%\addlinespace[2pt]
MaskCut \cite{wang2023cut}\textcolor{gray}{$_\text{ CVPR'23}$}      & 53.8 & 43.4 & 30.1 & 41.7 & 18.7 & 35.7 \\
%\addlinespace[2pt]
iSeg \cite{sun2024iseg}\textcolor{gray}{$_\text{ Arxiv'24}$}         & \xmark & \xmark & \xmark & 45.2 & 25.0 & \xmark \\
%\addlinespace[2pt]
DiffSeg \cite{tian2024diffuse}\textcolor{gray}{$_\text{ CVPR'24}$}      & 49.8 & 48.8 & 23.2 & 44.2 & 16.8 & 37.7 \\
%\addlinespace[2pt]
DiffCut \cite{couairon2024diffcut}\textcolor{gray}{$_\text{ NeurIPS'24}$} &
\underline{65.2} & \underline{56.5} & 34.1 &
49.1 & \underline{30.6} & 44.3 \\
%\addlinespace[2pt]
Seg4Diff \cite{kim2025seg4diff}\textcolor{gray}{$_\text{ NeurIPS'25}$} &
54.9 & 52.6 & \underline{38.5} & \underline{49.7} & 24.2 & \underline{44.9} \\
\midrule
\textbf{\ours{}} &
\textbf{66.9}\better{1.7} & \textbf{57.7}\better{1.2} & \textbf{42.7}\better{4.2} &
\textbf{52.0}\better{2.3} & \textbf{33.9}\better{3.5} & \textbf{48.0}\better{3.1} \\
\bottomrule
\end{tabular}
} % end resizebox
\end{table*}

\section{Experiments}

\subsection{Setup}

\mypar{Datasets.} 
We evaluate our approach across a diverse suite of semantic segmentation benchmarks. \emph{PASCAL~VOC} \cite{voc12} provides high-quality, object-centric images annotated over 20 foreground categories, serving as a classical benchmark for assessing general object segmentation. \emph{PASCAL~Context} \cite{pascalcontext} extends this setting by offering dense, scene-level annotations with a significantly larger label space. To evaluate performance on large-scale, open-domain visual concepts, we include \emph{COCO-Object} \cite{actualcoco}, which focuses on object regions, and \emph{COCO-Stuff-27} \cite{coco}, a compact variant emphasizing 27 coarse-grained ``stuff'' categories that capture global scene layout. We further assess our method on \emph{Cityscapes} \cite{cityscapes}, a real-world urban-driving dataset with fine-grained pixel annotations designed for structured street-scene understanding. Finally, \emph{ADE20K} \cite{ade20k} provides one of the most diverse and challenging segmentation environments, featuring 150 object and stuff classes across complex indoor and outdoor scenes. Together, these datasets form a comprehensive evaluation suite spanning object-centric, stuff-centric, and real-world structured environments.

\mypar{Baselines.} 
We compare \ours{} against a comprehensive set of recent training-free and weakly supervised open-vocabulary segmentation methods. \emph{ReCO}~\cite{shin2022reco} performs CLIP-guided region contrast to refine pseudo-masks from textual prompts. \emph{MaskCLIP}~\cite{ding2023maskclip} extends CLIP to pixel-level reasoning through feature masking and token-level activation maps. \emph{MaskCut}~\cite{wang2023cut} introduces a foreground-background separation strategy based on CLIP similarity and graph-cut refinement. \emph{iSeg}~\cite{sun2024iseg} proposes attention-driven affinity propagation using Stable Diffusion features for high-quality segmentation without training. \emph{DiffSeg}~\cite{tian2024diffuse} leverages diffusion model attention priors to extract dense semantic cues from generative backbones. \emph{DiffCut}~\cite{couairon2024diffcut} formulates segmentation as a diffusion-guided spectral partitioning problem, producing reliable object proposals from cross-attention structure. Finally, the recent \emph{Seg4Diff}~\cite{kim2025seg4diff} studies multi-modal diffusion transformers, identifying critical layers that yield high-quality segmentation proposals for training-free unsupervised segmentation. Together, these baselines represent the state of the art in training-free unsupervised segmentation, offering a strong comparison point for evaluating our method.

\mypar{Evaluation Metric.}We evaluate segmentation quality using mean Intersection-over-Union (mIoU), computed as the average IoU across all semantic classes. For a given class $c$, IoU is defined as
\begin{equation}
\mathrm{IoU}_c = \frac{|P_c \cap G_c|}{|P_c \cup G_c|},
\end{equation}
where $P_c$ and $G_c$ denote the predicted and ground-truth pixel sets for class $c$, respectively. The final mIoU score is obtained by averaging $\mathrm{IoU}_c$ over all classes.

\mypar{Implementation Details.} \ours{} is built on top of SSD-1B \cite{gupta2024progressive}, a distilled variant of Stable Diffusion XL~\cite{podell2023sdxl}. 
For all experiments, the model is conditioned on an empty text prompt and uses a fixed denoising timestep of 
\( t\,{=}\,50 \). Following prior work, we apply PAMR~\cite{araslanov2020single} as a post-processing refinement step to improve 
the quality of the predicted segmentation masks. Note that PAMR is applied on top of every method in our empirical comparisons. Our full pipeline runs on a single NVIDIA RTX A6000 (48\,GB) 
and processes input images of size \(1024\,{\times}\,1024\). For all experiments, we use an expansion coefficient of $\ell\,{=}\,2$ and a small pruning threshold of $\tau\,{=}\,1e{-}7$. The impact of other parameters is explored in Section \ref{sec:ablation}.

\subsection{Main Results}

\noindent \textbf{Quantitative evaluation.} As shown in Table~\ref{tab:main_result}, our proposed \emph{\ours{}} yields state-of-the-art zero-shot segmentation performance across all six benchmarks, outperforming all existing \textit{training-free} methods, by a clear margin. Older baselines such as ReCO, MaskCLIP, and MaskCut often struggle to recover complete object structures, while diffusion-based methods like DiffSeg and DiffCut benefit from strong semantic cues, yet remain constrained by the normalized graph cut formulation, which tends to oversmooth boundaries and erases fine-grained details. This leads to suboptimal segmentation results even for recent state-of-the-art approaches, 
particularly in datasets where precise boundary localization and object completeness are critical, e.g., DiffSeg in COCO-Object, Cityscapes, and ADE20K or Seg4Diff in Cityscapes. 
In contrast, the stochastic random-walk propagation strategy introduced in \emph{\ours{}} 
strengthens high-confidence affinity paths and suppresses noisy transitions, yielding more stable and spatially coherent segmentation masks. This results in consistent improvements across all datasets over very recent methods, such as DiffCut and Seg4Diff. For example, compared to the second best approach, \emph{\ours{}} yields improvements of \betterrr{1.7} on VOC, \betterrr{1.2} on Context, \betterrr{4.2} on COCO-Object, \betterrr{2.3} on COCO-Stuff-27, \betterrr{3.5} on Cityscapes, and \betterrr{3.1} on ADE20K. These gains demonstrate that diffusion affinity alone is insufficient without an effective propagation mechanism, firmly establishing \emph{\ours{}} as a more robust and reliable training-free alternative to current SOTA.

\noindent \textbf{Quantitative Analysis.} Fig.~\ref{fig:main-viz} provides a qualitative comparison with the recent DiffCut approach\footnote{From this point on, we adopt DiffCut as our main SoTA reference, since it achieves the second-best overall performance.}. %than Seg4Diff.}. 
From these visualizations, we observe that DiffCut tends to produce coarser and more uniform segmentations, often blending large regions such as sky, land, and sea into a single cluster. In contrast, \ours{} yields more structured and spatially coherent segments, particularly along object boundaries and fine-grained regions such as shorelines and building facades. Specifically, DiffCut frequently confuses semantically distinct areas (e.g., sky with background terrain), fails to capture smaller or localized structures such as the tree in the third example, and struggles to separate multiple meaningful objects within a scene. Moreover, in the portrait example, DiffCut misses fine details such as the subject’s hair, which are correctly delineated by \ours{}. Similarly, in the stadium scene, DiffCut oversmooths the segmentation and fails to isolate distinct structural components, while \ours{} preserves clearer object-level separation. Finally, in the indoor scene, DiffCut blends the bed and wall into a single region, whereas \ours{} produces more accurate and meaningful partitions. Overall, the visualization highlights \ours{}’s improved ability to preserve local details and capture meaningful spatial organization compared to DiffCut.

\begin{figure*}[t]
    \centering
    \includegraphics[width=\textwidth]{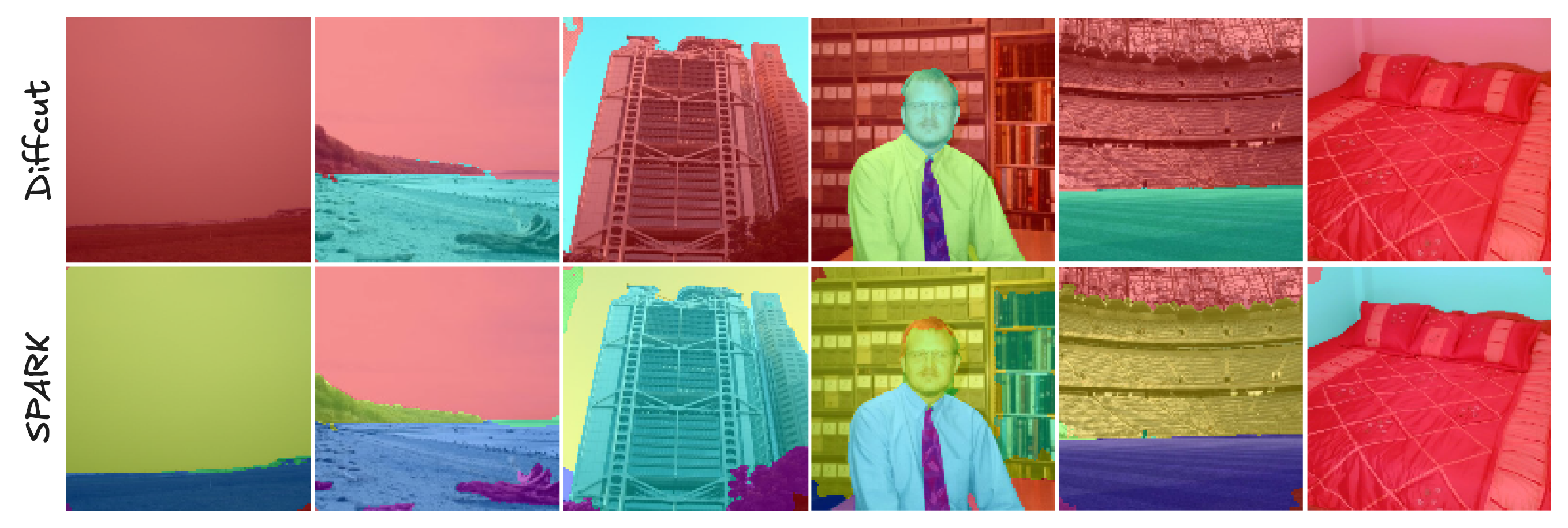}

    \vspace{-5pt}
    \caption{\textbf{\ours{} vs DiffCut} Qualitative comparison of segmentation maps produced by DiffCut (top row) and \ours{} (bottom row) on a variety of outdoor and indoor scenes. SPARK produces more spatially coherent and fine-grained segments, particularly along object boundaries.}
    \label{fig:main-viz}
    \vspace{-5pt}
\end{figure*}

\subsection{Ablation Study}\label{sec:ablation}
In this section, we perform a series of ablation studies to analyze the proposed approach, and the impact of some of its components on segmentation performance.

\mypar{Analysis for global and local affinity map (Table~\ref{table:ablations}).} While integrating the global affinity matrix typically achieves superior performance than DiffCut, incorporating local affinity
leads to largest gains, with consistent improvements across all benchmarks. More specifically, compared to the \textit{\ours{} w/ Global Affinity} version, these gains range from \betterrr{0.8} (Context) to \betterrr{6.4} (COCO-Object). 
These results demonstrate that
combining global and local affinity yields consistent accuracy gains across datasets, confirming the complementary nature of the two components within the diffusion affinity framework.

\mypar{Impact of post-processing, i.e., PAMR (Fig.~\ref{fig:pamr}).} A common strategy in the training-free unsupervised segmentation literature is to refine segmentation masks generated by the different methods with PAMR \cite{araslanov2020single}, as a post-processing step. To isolate the contribution of our core method, we now evaluate performance of \ours{} both with and without PAMR. As shown in Fig.~\ref{fig:pamr}, applying PAMR consistently improves mIoU by 2–5\% across all six benchmarks. Nevertheless, a key observation is that our raw predictions (i.e., without any post-processing) already match or exceed the PAMR-refined outputs of prior state-of-the-art methods (Table~\ref{tab:main_result}). This suggests that our approach produces inherently high-quality masks, reducing reliance on external refinement and highlighting its robustness in zero-shot settings. 

\begin{figure}[h!]
    \centering
    \includegraphics[width=\linewidth]{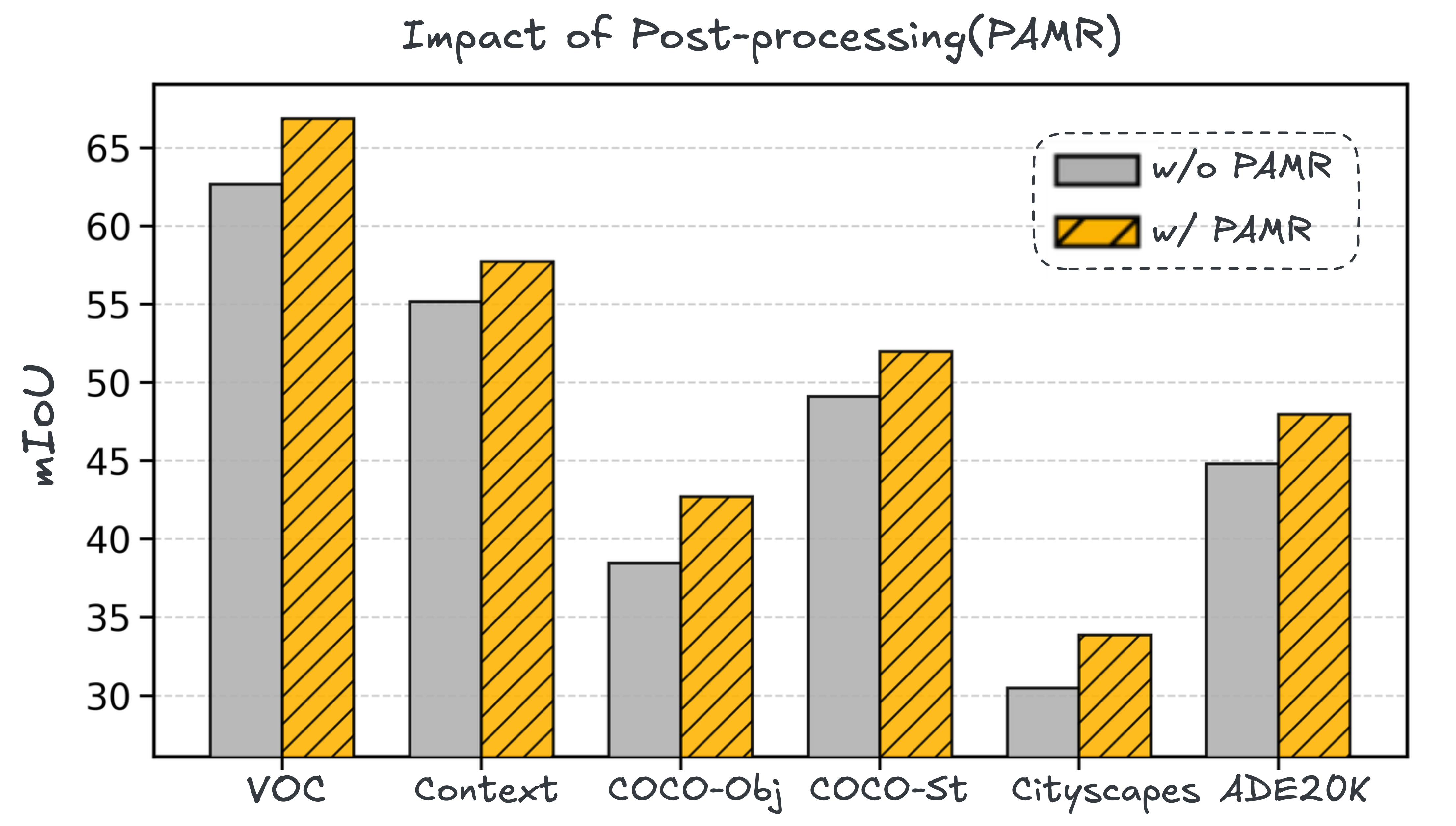}

    \vspace{-5pt}
    \caption{\textbf{Ablation of post-processing with PAMR.}
     We compare mIoU before and after applying PAMR on the validation splits of each dataset.}
    \label{fig:pamr}

    \vspace{-5pt}
\end{figure}

\mypar{Ablation for inflation factor.}
Figure~\ref{fig:inflation} presents the ablation study on the \emph{inflation factor}, evaluated on ADE20K. As the inflation factor increases from 1.5 to 2.4, mIoU improves steadily, indicating that progressively stronger cluster expansion enhances region consistency and label propagation. 
The highest performance is observed at an inflation factor of 2.6, suggesting an optimal balance between encouraging intra-region coherence and avoiding excessive fragmentation. Beyond this point, further increasing the inflation factor leads to a slight but consistent decline in mIoU, likely due to over-segmentation effects that disrupt semantic continuity. 
Overall, these results demonstrate that moderate inflation yields the best results, remaining stable across nearby values. 

\begin{figure}[t]
    \centering
    \includegraphics[width=.98\linewidth]{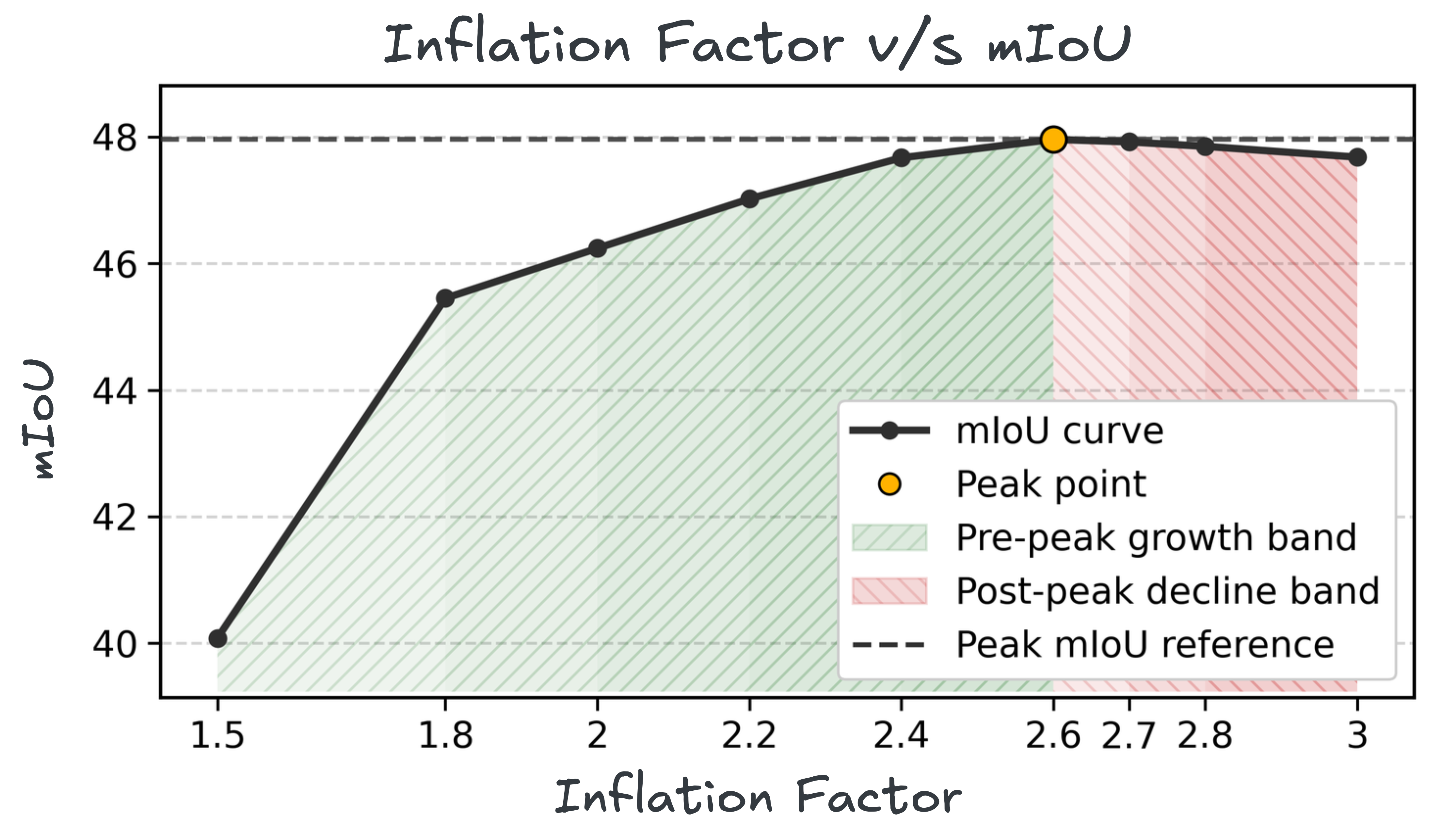}

    \vspace{-5pt}
    \caption{
    \textbf{Effect of the inflation factor on segmentation performance,} %(mIoU), 
    mIoU on ADE20K evaluated across a range of values, with the peak region highlighted and a dashed line indicating the maximum mIoU level.}
    \label{fig:inflation}

    \vspace{-5pt}
\end{figure}

\mypar{Sensitivity to $\boldsymbol{\beta}$.}
Figure~\ref{fig:beta} reports the ablation results for the scaling parameter $\beta$ on ADE20K. 
Performance remains largely stable for smaller values of $\beta$ in the range $[0, 0.4]$, 
indicating that weak scaling has limited impact on the overall behavior of the method. 
As $\beta$ increases, a consistent improvement is observed, with the highest mIoU attained at $\beta\,{=}\,0.6$, 
suggesting that moderate scaling enables more effective information propagation. 
Beyond this regime, larger values of $\beta$ lead to a gradual performance drop, likely due to excessive emphasis on propagated affinities, 
which can suppress discriminative local cues. 
An extreme degradation at $\beta\,{=}\,1.0$ indicates instability when the scaling fully dominates the update. 
Overall, these results show that $\beta$ provides meaningful gains when properly balanced, 
highlighting the importance of careful hyperparameter tuning.

\begin{figure}[h!]
    \centering
    \includegraphics[width=\linewidth]{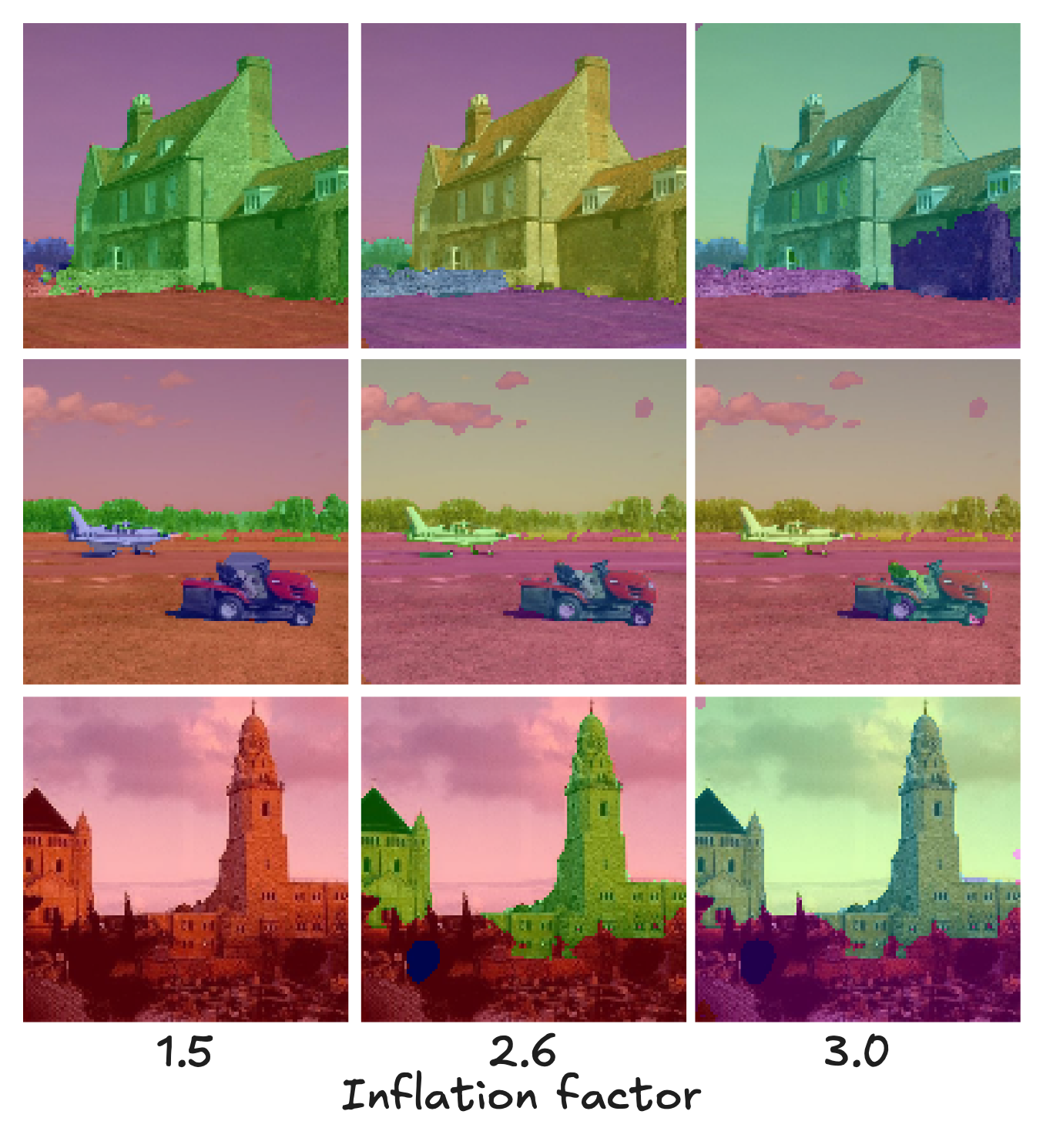}

    \vspace{-5pt}
    \caption{\textbf{Qualitative ablation for \emph{Inflation factor}.}}
    \label{fig:infa-viz}

    \vspace{-5pt}
\end{figure}

\begin{table*}[h!]
\scriptsize
\centering
\caption{\textbf{Ablation study of diffusion affinity components.} We evaluate the contribution of global and local affinity in \ours{} across six benchmarks. Incorporating local affinity consistently improves performance over both DiffCut and the global-only variant of \ours{}.}

\vspace{0.4em}
\label{table:ablations}

\resizebox{\textwidth}{!}{
\def\arraystretch{1.2}
\begin{tabular}{l|ccc|ccc}

\toprule
Affinity / Model & \textbf{VOC} & \textbf{Context} & \textbf{COCO-Object}
                 & \textbf{COCO-Stuff-27} & \textbf{Cityscapes} & \textbf{ADE20K} \\
\midrule

DiffCut  
& 65.2 & 56.5 & 34.1 & 49.1 & 30.6 & 44.3 \\

{\ours{} w/ Global Affinity}  
& {66.0}\better{0.8} & {56.9}\better{0.4} & {36.3}\better{2.2}
& {51.7}\better{2.6} & {30.3}\worse{0.3} & {44.1}\better{0.2} \\
%\addlinespace[2pt]

\midrule
% \rowcolor{brightgray}
\textbf{\ours{} w/ Global + Local}  
& \textbf{66.9}\better{1.7} & \textbf{57.7}\better{1.2} & \textbf{42.7}\better{8.6} &
\textbf{52.0.}\better{2.9} & \textbf{33.9}\better{3.5} & \textbf{48.0}\better{3.7} \\
%\addlinespace[2pt]
\bottomrule
\end{tabular}
}
\end{table*}

\mypar{Qualitative Analysis of Inflation factor.}In addition to the numerical trends, the qualitative results in Figure~\ref{fig:infa-viz} further illustrate the impact of the inflation factor on segmentation behavior. At lower values (e.g., 1.5), the model exhibits under-segmentation, where large regions are overly merged and object boundaries are poorly distinguished. As the inflation factor increases, the segmentation becomes progressively more refined, with improved separation of semantically meaningful regions and clearer object contours. At the optimal setting (2.6), the segmentation achieves a balanced structure, preserving intra-region coherence while accurately delineating distinct objects. However, for higher values (e.g., 3.0), the segmentation begins to fragment, introducing unnecessary splits within coherent regions, which is consistent with the observed decline in mIoU. These visual patterns closely align with the quantitative analysis, confirming that moderate inflation promotes the best trade-off between region expansion and fragmentation.

\begin{figure}[h!]
    \centering
    \includegraphics[width=\linewidth]{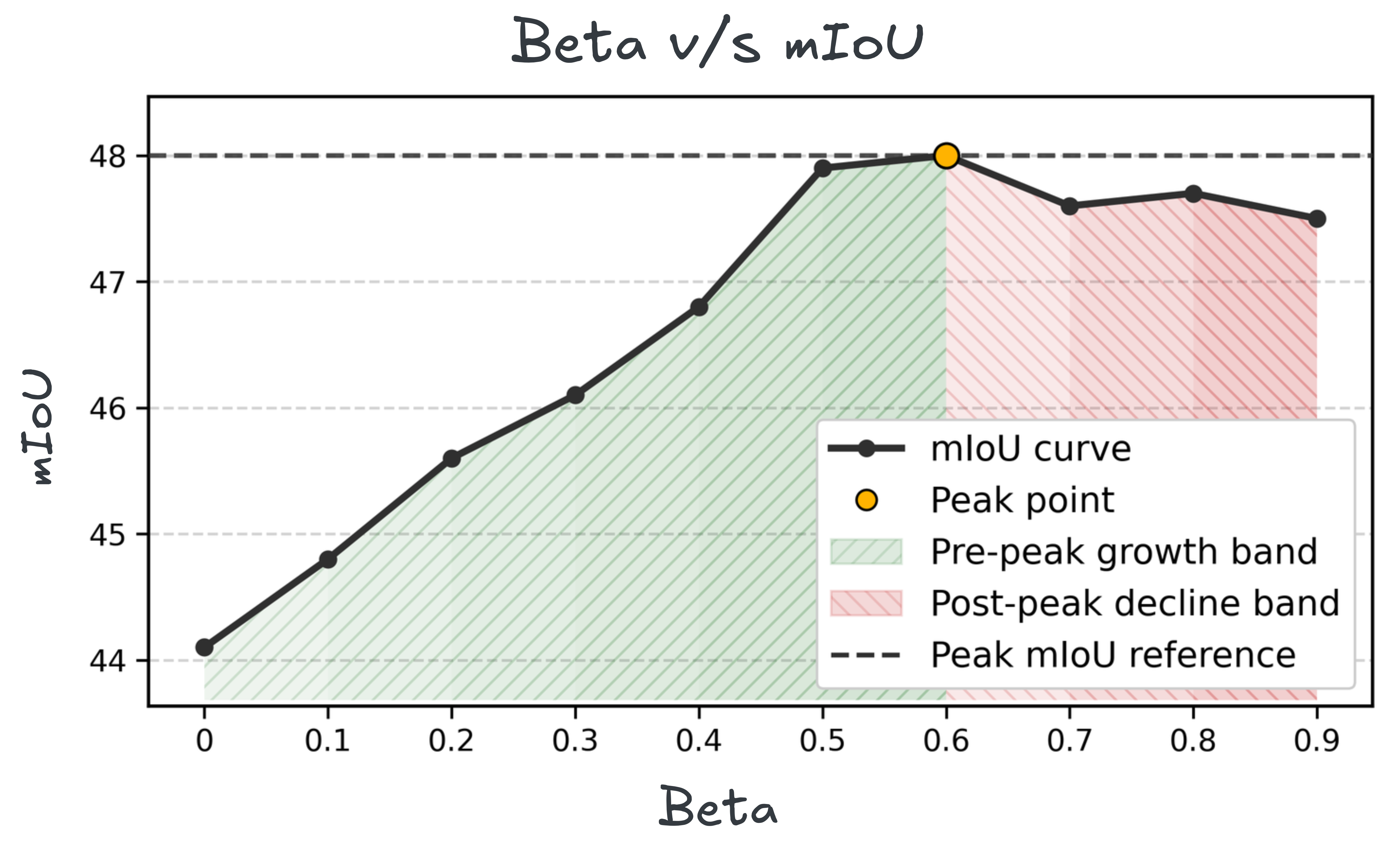}

    \vspace{-5pt}
    \caption{\textbf{Sensitivity to $\boldsymbol{\beta}$.}
    mIoU on ADE20K evaluated across different values of the $\beta$ hyperparameter, with bars grouped to indicate pre-peak, peak-region, and post-peak settings, and a dashed line marking the reference mIoU level.}
    \label{fig:beta}

    \vspace{-5pt}
\end{figure}

\section{Conclusion}
In this work, we revisited training-free segmentation through the lens of diffusion models and argued that the dominant spectral formulation adopted by prior methods is fundamentally misaligned with the nature of the problem. While diffusion attention provides rich semantic affinity, treating segmentation as a static graph partitioning objective leads to inherent limitations, including boundary oversmoothing, instability under noise, and a bias toward globally balanced partitions.
To address this, we proposed a new perspective: segmentation as a stochastic flow equilibrium problem. Under this formulation, meaningful regions are defined not by eigenvector geometry, but by how probability mass propagates and stabilizes over an affinity graph. This shift enables segmentation to emerge from the dynamics of a stochastic system, where segments correspond to flow-preserving ergodic components rather than solutions to a spectral objective. Based on this formulation, we introduced \ours{}, a training-free framework that integrates global diffusion attention with local spatial structure through stochastic propagation, and provided theoretical guarantees on convergence to a stable segmentation equilibrium.
Our results suggest that stochastic flow dynamics offer a principled alternative to spectral partitioning for training-free segmentation, and more broadly point toward dynamical systems as a robust foundation for unsupervised visual grouping.

\noindent \textbf{Limitations.} Our method relies on diffusion-based features and attention maps, and its performance is therefore bounded by the representational capacity of the underlying generative model. 
Furthermore, the current formulation is restricted to static images and does not explicitly model temporal consistency in video settings, making extensions to spatiotemporal segmentation a promising direction for future work. Finally, since the method builds upon pretrained generative models, potential biases inherited from these models may propagate into the segmentation outputs, which should be carefully considered when applying the approach in real-world or safety-critical scenarios.

\noindent \textbf{Impact Statement.} This work enables training-free semantic segmentation without requiring labeled data, facilitating broader adoption in data-scarce visual domains. More broadly, it introduces a stochastic dynamical perspective on segmentation that may inspire new research directions beyond spectral graph-based formulations.

\bibliography{main}
\bibliographystyle{icml2026}

% %%%%%%%%%%%%%%%%%%%%%%%%%%%%%%%%%%%%%%%%%%%%%%%%%%%%%%%%%%%%%%%%%%%%%%%%%%%%%%%
% %%%%%%%%%%%%%%%%%%%%%%%%%%%%%%%%%%%%%%%%%%%%%%%%%%%%%%%%%%%%%%%%%%%%%%%%%%%%%%%
% % APPENDIX
% %%%%%%%%%%%%%%%%%%%%%%%%%%%%%%%%%%%%%%%%%%%%%%%%%%%%%%%%%%%%%%%%%%%%%%%%%%%%%%%
% %%%%%%%%%%%%%%%%%%%%%%%%%%%%%%%%%%%%%%%%%%%%%%%%%%%%%%%%%%%%%%%%%%%%%%%%%%%%%%%
 \newpage
 \appendix
 \onecolumn


% You can have as much text here as you want. The main body must be at most $8$
% pages long. For the final version, one more page can be added. If you want, you
% can use an appendix like this one.

% The $\mathtt{\backslash onecolumn}$ command above can be kept in place if you
% prefer a one-column appendix, or can be removed if you prefer a two-column
% appendix.  Apart from this possible change, the style (font size, spacing,
% margins, page numbering, etc.) should be kept the same as the main body.
%%%%%%%%%%%%%%%%%%%%%%%%%%%%%%%%%%%%%%%%%%%%%%%%%%%%%%%%%%%%%%%%%%%%%%%%%%%%%%%
%%%%%%%%%%%%%%%%%%%%%%%%%%%%%%%%%%%%%%%%%%%%%%%%%%%%%%%%%%%%%%%%%%%%%%%%%%%%%%%

\section{Convergence Analysis of Markov Clustering}
\label{app:mcl}

We analyze the convergence of the Markov Clustering process using Birkhoff’s contraction theory, focusing on the expansion and inflation operators. For clarity of exposition, we omit the pruning step, which primarily accelerates convergence and does not affect the existence of the limit.

\mypar{Preliminaries.} For vectors $x,y \in \mathbb{R}^n_{>0}$, Hilbert's projective metric is defined as
\[
d_H(x,y)
=
\log\!\left(
\frac{\max_i x_i/y_i}{\min_i x_i/y_i}
\right).
\]
This metric is scale-invariant and is well-defined on the interior of the probability simplex.

For an order-preserving, homogeneous map $T : \mathbb{R}^n_{>0} \to \mathbb{R}^n_{>0}$, the
\emph{Birkhoff contraction coefficient} is
\[
\tau(T)
=
\sup_{x \neq y}
\frac{d_H(Tx,Ty)}{d_H(x,y)}.
\]
The \emph{projective diameter} of $T$ is
\[
\Delta(T)
=
\sup_{x,y \in \mathbb{R}^n_{>0}} d_H(Tx,Ty).
\]

A fundamental result due to Birkhoff~\citep{birkhoff1957} and Bushell~\citep{bushell1973} states that
if $\Delta(T) < \infty$, then
\[
\tau(T)
\le
\tanh\!\left(\frac{\Delta(T)}{4}\right).
\]

\begin{theorem}[Contraction of the Markov Clustering Operator]
Let $M \in \mathbb{R}^{n\times n}$ be a row-stochastic, primitive matrix and let $\ell \ge 2$. 
Define the Markov Clustering operator as
\[
T(M) := \Gamma_r(M^\ell),
\qquad
[\Gamma_r(X)]_{ij}
=
\frac{X_{ij}^r}{\sum_{j'} X_{i,j'}^r},
\quad r>1.
\]
where $M^\ell$ corresponds to the expansion step and $\Gamma_r$ to the inflation step. Then $T$ is a strict contraction with respect to Hilbert's projective metric. In particular,
\[
\tau(T)
\,\le\,
\tanh\!\left(\frac{\log r}{4}\right)
\,<\, 1,
\]
and the iterates $M_{t+1} = T(M_t)$ converge to a fixed point.
\end{theorem}

\begin{lemma}[Non-expansiveness of Expansion]
Let $A$ be a positive linear operator. Then
\[
d_H(Ax,Ay) \le d_H(x,y)
\quad \text{for all } x,y \in \mathbb{R}^n_{>0}.
\]
Consequently, for any $\ell \ge 1$,
\[
\tau(M^\ell) \le 1.
\]
\end{lemma}

\begin{proof}
This is a classical result due to Birkhoff~\citep{birkhoff1957} and Bushell~\citep{bushell1973}: positive linear operators are non-expansive in Hilbert's projective metric. Since $M$ is primitive, $M^\ell$ is positive for all sufficiently large $\ell$, and the bound follows.
\end{proof}

\begin{lemma}[Projective Diameter of Inflation]
For the inflation map $\Gamma_r$ with $r>1$, the projective diameter
\[
\Delta(\Gamma_r)
:=
\sup_{x,y \in \mathbb{R}^n_{>0}}
d_H(\Gamma_r(x),\Gamma_r(y))
\]
is finite and satisfies
\[
\Delta(\Gamma_r) = \log r.
\]
\end{lemma}

\begin{proof}
For $x,y>0$,
\[
\frac{\Gamma_r(x)_i}{\Gamma_r(y)_i}
=
\left(\frac{x_i}{y_i}\right)^r
\frac{\sum_j y_j^r}{\sum_j x_j^r}.
\]
Bounding the normalization factor using
\(
\min_i \tfrac{x_i}{y_i} \le \tfrac{x_j}{y_j} \le \max_i \tfrac{x_i}{y_i}
\)
yields uniform upper and lower bounds on these ratios, implying finiteness of the projective diameter. A sharp evaluation of the extremal configuration (Bushell, 1973) gives $\Delta(\Gamma_r)=\log r$.
\end{proof}

\begin{lemma}[Contraction of Inflation]
For all $r>1$,
\[
\tau(\Gamma_r)
\le
\tanh\!\left(\frac{\log r}{4}\right)
<1.
\]
\end{lemma}

\begin{proof}
By Birkhoff's inequality~\citep{birkhoff1957,bushell1973}, any order-preserving homogeneous map $T$ with finite projective diameter satisfies
\[
\tau(T)
\le
\tanh\!\left(\frac{\Delta(T)}{4}\right).
\]
Applying this inequality to $\Gamma_r$ and using Lemma~2 yields the result.
\end{proof}

\begin{proof}[Proof of Theorem~1]
Consider the composition $T = \Gamma_r \circ M^\ell$. 
By multiplicativity of Birkhoff contraction coefficients,
\[
\tau(T)
\le
\tau(\Gamma_r)\,\tau(M^\ell).
\]
By Lemma~1, $\tau(M^\ell)\le1$, and by Lemma~3,
\[
\tau(T)
\le
\tanh\!\left(\frac{\log r}{4}\right)
<1
\quad (r>1).
\]
Thus $T$ is a strict contraction in Hilbert's projective metric. Since the space of positive column-stochastic matrices is complete modulo scaling, Banach's fixed-point theorem implies existence, uniqueness (up to permutation), and convergence of the iterates.
\end{proof}

\end{document}